\documentclass{article}

\usepackage{PRIMEarxiv}

\usepackage[utf8]{inputenc} 
\usepackage[T1]{fontenc}    
\usepackage{hyperref}       
\usepackage{url}            
\usepackage{booktabs}       
\usepackage{amsfonts}       
\usepackage{nicefrac}       
\usepackage{microtype}      
\usepackage{fancyhdr}       
\usepackage{graphicx}       
\graphicspath{{figures/}}   

\usepackage{amsmath} 
\usepackage{amsthm}
\usepackage{amssymb}
\usepackage{longtable}
\usepackage{mathtools}
\usepackage{pdflscape}
\newtheorem{theorem}{Theorem}[section]
\newtheorem{corollary}{Corollary}[theorem]
\newtheorem{lemma}[theorem]{Lemma}
\newtheorem{proposition}[theorem]{Proposition}
\newtheorem{definition}{Definition}

\newtheorem{remark}{Remark}
\usepackage{algorithm}
\usepackage{algpseudocode}
\usepackage{natbib}

\title{Bayes-Sufficient Representations in Supervised Learning
}

\author{
  Vasileios Sevetlidis\\
  Athena Research Center, Kimmeria Campus, Xanthi, Greece \\
  Democritus University of Thrace, Vas. Sofias Campus, Xanthi, Greece \\
  International Hellenic University, Serres, Greece\\
  \texttt{vasiseve@athenarc.gr} \\
}

\begin{document}
\maketitle

\begin{abstract}
Representation learning is often described as preserving the information in an input that is relevant for prediction. This work asks what relevance means for a fixed supervised decision problem. A representation is defined to be Bayes-sufficient for a joint distribution and loss if some prediction head can use it to implement a Bayes-optimal action rule. This makes the target information loss-dependent. In the almost-surely unique Bayes-action case, the relevant object is a Bayes quotient, which identifies inputs that require the same Bayes-optimal action. A representation is sufficient when it refines this quotient, and Bayes-minimal when it is informationally equivalent to it. The framework connects naturally to property elicitation: zero-one loss requires the Bayes class, squared loss the conditional mean, Brier loss the conditional probability in binary prediction, and log loss or strictly proper scoring rules the predictive distribution. Controlled finite experiments, learned neural bottleneck experiments, and a real-data iNaturalist taxonomic refinement experiment illustrate the distinction between sufficiency, minimality, and retained non-required information. For a fixed supervised problem, the distribution and the loss determine the Bayes action, the Bayes action determines the quotient, and the quotient determines the minimal information required for Bayes-optimal prediction.
\end{abstract}


\section{Introduction}
\label{sec:introduction}

Representation learning is often described as the problem of retaining the information in an input that is relevant for prediction while discarding nuisance variation. This description is appealing, but it leaves open a basic question: relevant for which prediction problem? Relevance is not a property of the input distribution alone. It is determined jointly by the data-generating law and the loss under which predictions are evaluated.

A simple example illustrates the issue. In binary classification, two inputs may have conditional probabilities $P(Y=1\mid X=x)$ equal to $0.55$ and $0.95$. Under zero-one loss, both inputs require the same Bayes-optimal class decision, so a representation that records only the Bayes class is sufficient for optimal classification. Under log loss or another calibrated probabilistic loss, however, the two inputs cannot be identified; the probabilities themselves are the quantities to be reported. The same conditional law can therefore induce a coarse representation target for classification and a finer one for probabilistic prediction.

This paper formalizes this loss-dependence through the notion of a Bayes-sufficient representation. For a fixed joint distribution and supervised loss, a representation is Bayes-sufficient if some prediction head can use it to implement a Bayes-optimal action rule. In the common case where the Bayes action is almost surely unique, this definition yields a canonical loss-dependent object: the Bayes quotient. This quotient identifies inputs that require the same Bayes-optimal action. A representation is sufficient exactly when it refines the Bayes quotient, and it is Bayes-minimal when it is informationally equivalent to the quotient.

The resulting framework separates two notions that are often conflated. A representation may be sufficient because it contains all information needed for Bayes-optimal prediction, while still being nonminimal because it also retains information that the loss does not require. For example, under zero-one classification the relevant quotient is generated by the Bayes class, whereas under squared loss it is generated by the conditional mean, under binary Brier loss by the conditional probability, and under log loss or strictly proper scoring rules by the predictive distribution. These are different notions of prediction-relevant information, even when the underlying joint distribution is the same.

The paper makes three contributions. First, it gives a decision-theoretic definition of Bayes sufficiency for supervised representations and proves a factorization characterization: a representation is Bayes-sufficient exactly when at least one Bayes predictor is measurable with respect to the representation. Second, in the unique Bayes-action case, it identifies the Bayes quotient and characterizes sufficiency and minimality as sigma-algebra containment and equality. The set-valued, nonunique-action case is also treated, where sufficiency requires a common optimal action on representation fibers rather than equality of Bayes-action sets. Third, it connects the quotient view to property elicitation, showing that losses induce representation targets through the conditional properties they elicit.

The rest of the paper is organized as follows. Section~\ref{sec:related-work} situates Bayes sufficiency relative to statistical decision theory, classical sufficiency, Blackwell comparison, sufficient dimension reduction, property elicitation, information-bottleneck methods, and probing. Section~\ref{sec:theory} contains the main theoretical development and is the central section of the paper: it defines Bayes sufficiency, introduces the Bayes quotient in the unique-action case, distinguishes sufficiency from Bayes minimality, and derives the representation targets induced by standard supervised losses. Section~\ref{sec:experiments} provides empirical illustrations. Section~\ref{subsec:exp-synthetic} studies controlled synthetic settings where the relevant quotients are known exactly and then examines learned neural bottleneck representations. Section~\ref{subsec:exp-inaturalist} gives a real-data refinement experiment on iNaturalist, using the species--genus--family hierarchy as a structured analogue of quotient refinement. Section~\ref{sec:discussion} discusses limitations and extensions, and Section~\ref{sec:conclusion} concludes. The full measure-theoretic formulation, including technical conditions, proofs, the nonunique-action case, and the standard-loss derivations, is given in Appendix~\ref{app:bayes-quotient-theory}.

\section{Related work}
\label{sec:related-work}

The starting point is classical statistical decision theory. In the Wald--Bayes formulation, optimality is not a property of a distribution alone, but of a decision problem: an action space, a loss, and a probability law together determine the Bayes action \citep{Wald1950,Ferguson1967,Berger1985}. This elementary point is central for representation learning. If a representation is meant to preserve the information relevant for prediction, the word relevant cannot be interpreted independently of the loss. The same conditional law of $Y$ given $X$ may induce a coarse optimal action under one loss and a finer optimal action under another. Task-relevant information is loss-relative.

This notion of Bayes sufficiency differs fundamentally from classical statistical sufficiency. In the Fisher--Neyman sense, a statistic is sufficient for a parameterized model when it preserves the information in the sample about the unknown parameter, equivalently through the factorization criterion under suitable domination assumptions \citep{Fisher1922,Neyman1935,LehmannScheffe1950}. The present setting is not primarily about preserving information about a parameter or model family. The joint law $P$ is fixed, and the question is instead which information about the input must be retained so that a supervised Bayes action can still be implemented. Blackwell sufficiency is closer in decision-theoretic spirit, since it compares experiments by their value for decision making \citep{Blackwell1951,Blackwell1953}. However, Blackwell's order is deliberately uniform over classes of decision problems: one experiment is more informative than another if it can do at least as well across relevant losses, priors, and actions. The present work fixes a single supervised decision problem and extracts the induced representation quotient. It is narrower than Blackwell comparison, but aims at identifying the information demanded by a specified loss.

The work is also adjacent to sufficient dimension reduction in regression. SDR asks for low-dimensional summaries of predictors that preserve aspects of the conditional relationship between $X$ and $Y$, often through conditions such as $Y \perp X \mid B^\top X$ or through central subspaces for the conditional distribution or conditional mean \citep{Li1991,Cook1998,CookLi2002,FukumizuBachJordan2009}. These works are concerned with dimension reduction, estimation of subspaces, and regression structure. By contrast, the object here need not be a linear subspace or low-dimensional Euclidean coordinate. The Bayes quotient is an informational object. The sigma-algebra generated by the Bayes action, or more generally the common-action structure induced by the loss. SDR asks for summaries that preserve conditional structure; Bayes sufficiency asks which conditional property is required for optimal action under the chosen supervised loss.

The closest formal language for loss-dependent targets comes from property elicitation. A loss elicits a property of a distribution when its expected loss is uniquely minimized by reporting that property. Squared loss elicits the mean, absolute loss elicits a median, quantile losses elicit quantiles, Brier-type losses elicit probabilities, and strictly proper scoring rules elicit predictive distributions \citep{Savage1971,GneitingRaftery2007,LambertPennockShoham2008,FrongilloKash2015}. In supervised learning, the same statement applies conditionally: the loss determines which property of $P(Y\in\cdot\mid X)$ is optimal to report at each input. The present paper uses this observation at the representation level. Once a loss elicits a conditional property, the Bayes quotient is the input-level quotient generated by that property. Bayes-sufficient representations are precisely those from which the elicited conditional property, or at least an optimal action, can be recovered.

This perspective clarifies a recurring theme in representation learning. Methods such as the information bottleneck frame representation learning as preserving information about a relevance variable while compressing information about the input \citep{TishbyPereiraBialek1999}. Subsequent work connects minimality, invariance, nuisance removal, and deep representations \citep{TishbyZaslavsky2015,AchilleSoatto2018,AchilleSoatto2018InformationDropout}. These approaches motivate the idea that good representations should keep task-relevant information and discard nuisance variation. The present paper makes the supervised target of relevance explicit: for a fixed joint law and loss, the Bayes information is the part of the input information demanded by the decision problem. A representation can contain this information while also carrying extra information not used by the Bayes action.

The empirical methodology should be read with the same distinction in mind. Probing studies ask whether a variable is recoverable from a learned representation by a chosen class of probes, often linear probes \citep{AlainBengio2017,HewittLiang2019,Belinkov2021}. Such probes are valuable diagnostics, but they do not by themselves establish population sigma-algebra inclusion. A successful probe shows that the target variable is accessible to the probe class under the sample and optimization procedure used; a failed probe may indicate absence of information, insufficient probe capacity, finite-sample limitations, or optimization failure. For this reason, the experiments in this paper use probes as empirical recoverability tests rather than as definitions of sufficiency. Population Bayes sufficiency is a factorization property; probing is an operational diagnostic for learned representations.

Finally, this paper complements recent representation-level non-identifiability and invariance perspectives \citep{Sevetlidis2026Fiber,SevetlidisPavlidis2026Holonomy}. In particular, \citet{Sevetlidis2026Fiber} emphasizes that predictor behavior alone need not determine representation-level properties, since unused coordinates can be added or transformed without changing the composite predictor. Here the positive decision-theoretic question is: among all information that a representation could retain, which information is required by a specified supervised loss? The answer comes through the Bayes quotient. The fiber perspective explains why representation properties cannot be inferred from predictor behavior alone; Bayes sufficiency identifies the loss-dependent information that must be present for optimal supervised action.

\section{Theory}
\label{sec:theory}

This section characterizes the information a representation must retain to support Bayes-optimal prediction for a fixed supervised decision problem. The joint law of $(X,Y)$ and the loss determine the Bayes action, and hence the loss-dependent information required for optimal prediction. Writing predictors as $c\circ h$, a representation $h$ is Bayes-sufficient when some head $c$ can use it to implement a Bayes-optimal rule \citet{Sevetlidis2026Fiber}. When the Bayes action is almost surely unique, this information is captured by the Bayes quotient, which identifies inputs requiring the same Bayes-optimal action. In this case, sufficiency means refinement of the quotient, while Bayes minimality means informational equivalence to it. The Appendix~\ref{app:bayes-quotient-theory} gives the full measure-theoretic formulation, including the nonunique-action case.

\subsection{Bayes sufficiency}
\label{subsec:bayes-sufficiency-main}

Let $(X,Y)\sim P$ on standard Borel spaces $\mathcal X\times\mathcal Y$, and let $\ell:\mathsf A\times\mathcal Y\to[0,\infty]$ be a measurable loss with action space $\mathsf A$. A predictor is a measurable map $f:\mathcal X\to\mathsf A$, and a representation is a measurable map $h:\mathcal X\to\mathcal H$. Write $H=h(X)$. A head on the representation is a measurable map $c:\mathcal H\to\mathsf A$, giving the predictor $c\circ h$.

Let
\[
\mu_x(\cdot):=P(Y\in\cdot\mid X=x)
\]
be a regular conditional law. The conditional risk of action $a$ at input value $x$ is
\[
L_{\ell,P}(a\mid x):=\int_{\mathcal Y}\ell(a,y)\,\mu_x(dy),
\]
and the Bayes action correspondence is
\[
\Gamma_{\ell,P}(x):=\arg\min_{a\in\mathsf A}L_{\ell,P}(a\mid x).
\]
Throughout, Bayes actions are assumed to exist almost surely and the usual measurability conditions are assumed to hold; the precise conditions are collected in Appendix~\ref{app:technical-setup}.

\begin{definition}[Bayes sufficiency]
\label{def:bayes-sufficiency-main}
A representation $H=h(X)$ is \emph{Bayes-sufficient} for $(P,\ell)$ if there exists a measurable head $c:\mathcal H\to\mathsf A$ such that
\[
c(h(X))\in\Gamma_{\ell,P}(X)
\qquad P_X\text{-a.s.}
\]
Equivalently, at least one Bayes-optimal decision rule can be implemented from the representation.
\end{definition}

The definition is existential. In particular, when Bayes actions are nonunique, Bayes sufficiency asks for enough information to choose one Bayes-optimal action almost surely. The entire correspondence $x\mapsto\Gamma_{\ell,P}(x)$ may contain additional optimal actions that the representation leaves unspecified. For example, if
\[
\Gamma_{\ell,P}(x_1)=\{a,b\},
\qquad
\Gamma_{\ell,P}(x_2)=\{a\},
\]
then a representation may merge $x_1$ and $x_2$ and still support Bayes optimality, since a head can choose the common action $a$. A formal fiberwise characterization of this set-valued case is given in Appendix~\ref{app:nonunique-actions}.

The following elementary factorization principle is used throughout. Its proof is included in Appendix~\ref{app:factorization-proof}.

\begin{lemma}[Bayes sufficiency as factorization]
\label{lem:bayes-sufficiency-factorization-main}
A representation $H=h(X)$ is Bayes-sufficient for $(P,\ell)$ if and only if at least one Bayes predictor is $\sigma(H)$-measurable. Equivalently, there exist a Bayes predictor $f^\star$ and a measurable head $c:\mathcal H\to\mathsf A$ such that
\[
f^\star(X)=c(H)
\qquad P_X\text{-a.s.}
\]
\end{lemma}

Thus Bayes sufficiency asserts exactly that $H$ contains enough information to implement one Bayes-optimal action rule for the chosen loss.

\subsection{The Bayes quotient under unique Bayes actions}
\label{subsec:unique-bayes-quotient-main}

Now specialize to the common case in which the Bayes action is unique almost surely. Assume that there exists a measurable map
\[
a^\star_{\ell,P}:\mathcal X\to\mathsf A
\]
such that
\[
\Gamma_{\ell,P}(X)=\{a^\star_{\ell,P}(X)\}
\qquad P_X\text{-a.s.}
\]
Then every Bayes predictor agrees almost surely with $a^\star_{\ell,P}$. The random variable $a^\star_{\ell,P}(X)$ is the Bayes-relevant decision statistic for the loss.

\begin{definition}[Bayes quotient and Bayes information]
\label{def:bayes-quotient-main}
In the unique-action case, the \emph{Bayes information} of $(P,\ell)$ is
\[
\mathcal I_{\ell,P}:=\sigma(a^\star_{\ell,P}(X)).
\]
A \emph{Bayes quotient} is any random variable $Q_{\ell,P}=q_{\ell,P}(X)$ satisfying
\[
\sigma(Q_{\ell,P})=\mathcal I_{\ell,P}
\qquad \text{mod }P_X.
\]
\end{definition}

The quotient is meant informationally. Inputs are identified when they induce the same Bayes action. The formal object is the sigma-algebra generated by that action, which avoids choosing a particular quotient space and makes the construction invariant under measurable reparameterizations of the Bayes act.

\begin{theorem}[Bayes quotient characterization]
\label{thm:bayes-quotient-main}
Assume the Bayes action is unique $P_X$-almost surely and is represented by $a^\star_{\ell,P}$. Let $H=h(X)$. Then the following statements are equivalent.
\begin{enumerate}
    \item $H$ is Bayes-sufficient for $(P,\ell)$;
    \item there exists a measurable head $c:\mathcal H\to\mathsf A$ such that
    \[
    a^\star_{\ell,P}(X)=c(H)
    \qquad P_X\text{-a.s.};
    \]
    \item $a^\star_{\ell,P}(X)$ is $\sigma(H)$-measurable;
    \item the Bayes information is contained in the representation information,
    \[
    \mathcal I_{\ell,P}\subseteq\sigma(H)
    \qquad \text{mod }P_X.
    \]
\end{enumerate}
Consequently, in the unique-action case, Bayes-sufficient representations are exactly refinements of the Bayes quotient.
\end{theorem}

\begin{proof}
By Lemma~\ref{lem:bayes-sufficiency-factorization-main}, Bayes sufficiency is equivalent to the existence of a Bayes predictor that factors through $H$. Under uniqueness, every Bayes predictor agrees almost surely with $a^\star_{\ell,P}$. Therefore Bayes sufficiency is equivalent to the existence of a measurable $c:\mathcal H\to\mathsf A$ with $a^\star_{\ell,P}(X)=c(H)$ almost surely. By the Doob--Dynkin factorization lemma for standard Borel spaces, this is equivalent to $a^\star_{\ell,P}(X)$ being $\sigma(H)$-measurable, which is exactly the inclusion $\mathcal I_{\ell,P}\subseteq\sigma(H)$ modulo $P_X$.
\end{proof}

A pointwise intuition is useful. Up to null sets, a Bayes-sufficient representation may split Bayes-equivalence classes and must keep distinct input regions that require different unique Bayes actions. The representation $H$ may refine the Bayes quotient; sufficiency enforces retention of the quotient itself.

\subsection{Bayes minimality}
\label{subsec:bayes-minimality-main}

The quotient characterization separates sufficiency from minimality. Sufficiency requires that the representation contain the Bayes information. Minimality requires informational equivalence with the Bayes information.

\begin{definition}[Bayes minimality]
\label{def:bayes-minimality-main}
In the unique-action setting, a representation $H=h(X)$ is \emph{Bayes-minimal} for $(P,\ell)$ if
\[
\sigma(H)=\mathcal I_{\ell,P}
\qquad \text{mod }P_X.
\]
Equivalently, $H$ is informationally equivalent to the Bayes quotient $Q_{\ell,P}$.
\end{definition}

Thus
\[
H\text{ Bayes-sufficient}
\quad\Longleftrightarrow\quad
\sigma(Q_{\ell,P})\subseteq\sigma(H),
\]
whereas
\[
H\text{ Bayes-minimal}
\quad\Longleftrightarrow\quad
\sigma(H)=\sigma(Q_{\ell,P}).
\]
Bayes-minimal representations are therefore unique up to information equivalence. Any two of them generate the same sigma-algebra modulo $P_X$.

This distinction matters because sufficiency is preserved under refinement, and minimality imposes equality with the quotient. If $H$ is Bayes-sufficient and $G=g(X)$ satisfies $\sigma(H)\subseteq\sigma(G)$ modulo $P_X$, then $G$ is also Bayes-sufficient. If $H$ is Bayes-minimal and is augmented by an additional variable $V=v(X)$ outside the Bayes quotient, then $(H,V)$ remains Bayes-sufficient and becomes a nonminimal refinement. This is the Bayes-quotient version of the broader factorization phenomenon. Unused representation information may be retained while predictive capability is preserved.

\subsection{Connection with elicited conditional properties}
\label{subsec:elicitation-main}

Many supervised losses are designed to elicit a property of the conditional law of $Y$ given $X$. The Bayes quotient is then the input quotient generated by that conditional property.

Let $\mathcal P_0\subseteq\mathcal P(\mathcal Y)$ be a class of probability measures, and let
\[
T:\mathcal P_0\to\mathcal T
\]
be a measurable property. Suppose the action space is $\mathsf A=\mathcal T$. A loss $\ell:\mathcal T\times\mathcal Y\to[0,\infty]$ uniquely elicits $T$ on $\mathcal P_0$ if, for every $\mu\in\mathcal P_0$,
\[
\arg\min_{t\in\mathcal T}\int \ell(t,y)\,\mu(dy)=\{T(\mu)\}.
\]

\begin{proposition}[Elicited properties induce Bayes quotients]
\label{prop:elicited-quotient-main}
Assume that $P(Y\in\cdot\mid X)\in\mathcal P_0$ almost surely and that $T(P(Y\in\cdot\mid X))$ is measurable. If $\ell$ uniquely elicits $T$ on $\mathcal P_0$, then
\[
a^\star_{\ell,P}(X)=T(P(Y\in\cdot\mid X))
\]
almost surely. Consequently,
\[
H\text{ is Bayes-sufficient}
\quad\Longleftrightarrow\quad
\sigma\bigl(T(P(Y\in\cdot\mid X))\bigr)\subseteq\sigma(H)
\quad \text{mod }P_X,
\]
and $H$ is Bayes-minimal if and only if equality holds.
\end{proposition}

\begin{proof}
Conditioning on $X=x$, the expected loss of reporting $t$ is $\int\ell(t,y)\,\mu_x(dy)$. Unique elicitation makes $T(\mu_x)$ the unique minimizer for almost every $x$. The result follows from Theorem~\ref{thm:bayes-quotient-main}.
\end{proof}

\subsection{Standard supervised losses}
\label{subsec:standard-losses-main}

The preceding proposition recovers the usual supervised examples and shows why the representation target is loss-dependent. Different losses applied to the same joint law may require different functions of the conditional distribution.

\begin{center}
\begin{tabular}{lll}
\toprule
\textbf{Loss} & \textbf{Bayes action} & \textbf{Bayes quotient} \\
\midrule
Zero-one classification & $\arg\max_k P(Y=k\mid X)$ & Bayes class \\
Squared loss & $\mathbb E[Y\mid X]$ & Conditional mean \\
Binary Brier loss & $P(Y=1\mid X)$ & Conditional probability \\
Finite-label log loss & $(P(Y=1\mid X),\dots,P(Y=K\mid X))$ & Conditional class-probability vector \\
Strictly proper scoring rule & $P(Y\in\cdot\mid X)$ & Conditional law \\
\bottomrule
\end{tabular}
\end{center}

More explicitly, for zero-one classification with $\mathcal Y=\{1,\dots,K\}$, the Bayes quotient is generated by the Bayes class whenever the conditional class probabilities have a unique maximizer almost surely. In the binary case, writing $\eta(X)=P(Y=1\mid X)$ and assuming $P_X(\eta(X)=1/2)=0$, the quotient is generated by $\mathbf 1\{\eta(X)>1/2\}$.

For squared loss on $\mathbb R^d$, assuming $\mathbb E\|Y\|_2^2<\infty$, the Bayes action is the conditional mean $\mathbb E[Y\mid X]$. For binary Brier loss, the Bayes action is the conditional probability $P(Y=1\mid X)$. For finite-label log loss, the Bayes action is the full conditional probability vector. More generally, a strictly proper scoring rule elicits the conditional law itself, so its Bayes quotient is generated by $P(Y\in\cdot\mid X)$.

Thus the pair $(P,\ell)$ determines the minimal predictive representation. Log loss and strictly proper scoring rules require the conditional law, zero-one loss requires only the Bayes class, squared loss requires the conditional mean, and binary Brier loss requires the conditional probability.

\section{Experiments}
\label{sec:experiments}

The theory in Section~\ref{sec:theory} is a population theory: it gives conditions under which a representation contains enough information to implement Bayes-optimal decisions for a fixed loss. The experiments have a deliberately limited role. They make the quotient structure visible in settings where the decisive variables can be inspected.

The evidence is organized into two parts. Section~\ref{subsec:exp-synthetic} uses a synthetic distribution whose Bayes quotients are known analytically. Population risks are first computed for hand-defined representations; neural bottleneck representations are then learned from continuous observations generated from the same latent law. This separates the information-theoretic statement from the empirical question of what a finite trained encoder retains. Section~\ref{subsec:exp-inaturalist} moves to a real fine-grained vision dataset, iNaturalist, where the biological taxonomy provides a natural deterministic refinement from species to genus to family. That experiment gives a DNN analogue of the same refinement logic.

Throughout, population sufficiency is distinguished from empirical recoverability. In the finite synthetic calculation, sufficiency and minimality are known analytically. In the learned-representation experiments, the encoder is frozen and downstream probes are trained. Probe success shows that the target variable is accessible to the chosen probe class; probe failure is interpreted more cautiously, since information may be absent or encoded in a way beyond the chosen probe class.

\subsection{Known quotients and learned refinements}
\label{subsec:exp-synthetic}

The synthetic experiments put the theory in a setting where the quotients are known in advance. One latent law supports two supervised decision problems. For zero-one classification, the Bayes quotient is coarse. For probabilistic prediction under log loss or Brier score, the Bayes quotient is strictly finer. The point is twofold: Bayes sufficiency is loss-dependent, and sufficiency need not be minimality.

Let
\[
S\in\{-1,+1\},\qquad T\in\{1,\dots,5\},
\]
with $S$ and $T$ sampled independently and uniformly. Conditional on $(S,T)$,
\[
Y\mid S,T\sim\mathrm{Bernoulli}(\eta(S,T)).
\]
The conditional probabilities are
\[
\eta(+1,T)\in\{0.55,0.65,0.75,0.85,0.95\},
\]
with symmetric values
\[
\eta(-1,T)\in\{0.45,0.35,0.25,0.15,0.05\}.
\]
The Bayes classifier is determined by $S$, whereas calibrated probabilistic prediction requires the probability level $\eta(S,T)$. Since the ten values of $\eta(S,T)$ are distinct, the log-loss quotient is equivalent to the full pair $(S,T)$.

First, three hand-defined representations are compared.
\[
H_1=S,\qquad H_2=(S,T),\qquad H_3=(S,T,U),
\]
where $U$ is an independent irrelevant coordinate with seven possible values. Population risks are computed by summing over the finite support. The results are shown in Table~\ref{tab:synthetic-exact}. The representation $H_1=S$ is sufficient and minimal for zero-one classification, because it coincides with the Bayes-class quotient. It falls short for log loss or Brier score, because it collapses the five conditional-probability levels within each sign of $S$. The representation $H_2=(S,T)$ determines $\eta(S,T)$, making it sufficient and minimal for probabilistic prediction; for classification, it is a nonminimal refinement. Finally, $H_3=(S,T,U)$ is sufficient for the same losses as $H_2$ and nonminimal because it contains irrelevant information.

\begin{table}[t]
\centering
\caption{Known-quotient finite experiment. Population sufficiency/minimality and exact risks are computed by summing over the finite support. Empirical risks are downstream-head results at $n_{\mathrm{train}}=10{,}000$, averaged over repetitions.}
\label{tab:synthetic-exact}
\begin{tabular}{lcccccccc}
\toprule
Representation & Atoms & Acc. suff. & Log suff. & Min. acc. & Min. log & Pop. NLL & Emp. NLL & Emp. Brier \\
\midrule
$H_1=S$ & 2  & yes & no  & yes & no  & 0.5623 & $0.5624\pm0.0023$ & $0.1875\pm0.0010$ \\
$H_2=(S,T)$ & 10 & yes & yes & no  & yes & 0.5038 & $0.5043\pm0.0024$ & $0.1677\pm0.0009$ \\
$H_3=(S,T,U)$ & 70 & yes & yes & no & no & 0.5038 & $0.5073\pm0.0026$ & $0.1687\pm0.0010$ \\
\bottomrule
\end{tabular}
\end{table}

The empirical downstream heads reproduce the population pattern. Accuracy is essentially identical for $H_1$ and $H_2$, while $H_1$ has substantially larger NLL and Brier score. The representation $H_3$ is population-equivalent to $H_2$; its empirical NLL and Brier score are slightly worse at finite sample size because the probability head must estimate more cells. This finite-sample gap reflects the expected distinction between population information and finite-sample estimation.

Paired test-set contrasts for $H_1-H_2$ make the loss-dependence explicit. The zero-one gap is zero; the NLL gap is approximately $0.059$ with 95\% CI $[0.057,0.061]$ and the Brier gap is approximately $0.020$ with 95\% CI $[0.0195,0.0209]$. The coarse quotient is Bayes-sufficient for classification and falls short for probabilistic prediction.

\begin{figure}[t]
\centering
\includegraphics[width=.74\linewidth]{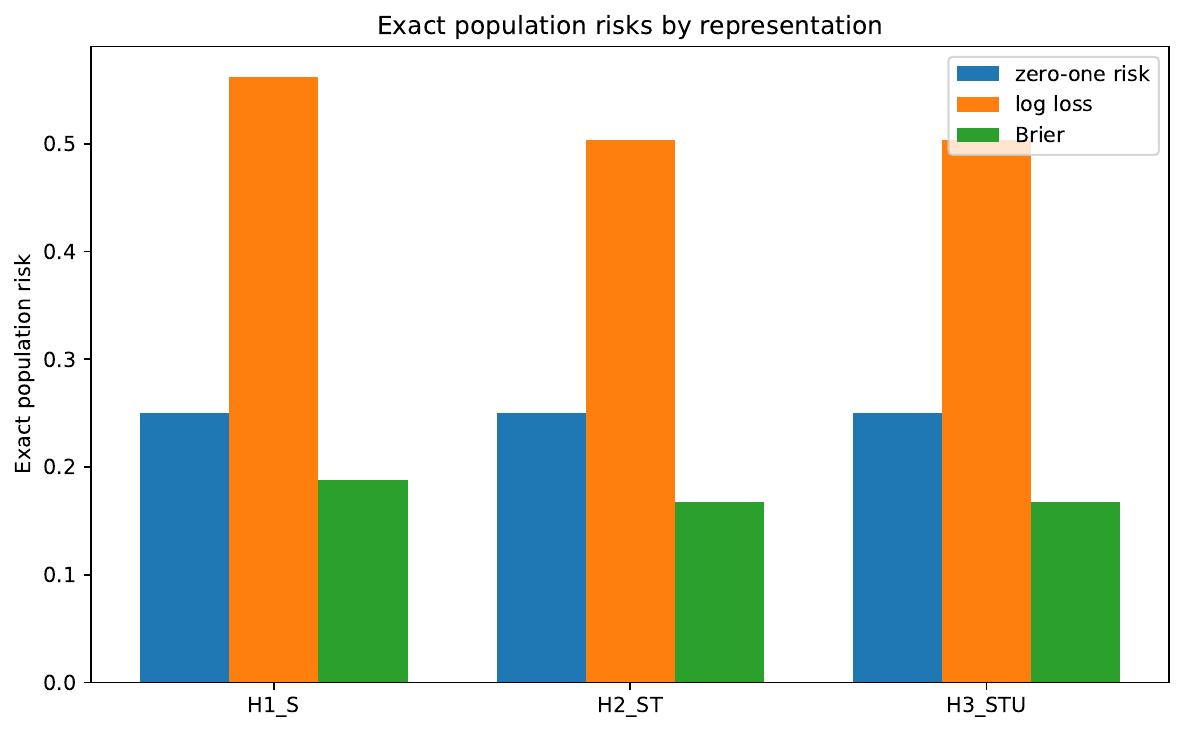}
\caption{Exact population risks in the known-quotient finite experiment. All three representations contain the Bayes class and hence have the same zero-one risk. Only representations that determine $(S,T)$ attain the optimal log-loss and Brier risks.}
\label{fig:synthetic-pop-risks}
\end{figure}

Next, neural encoders are learned from continuous observations generated by this latent law. The observed input is a continuous vector containing a strong component for $S$, weaker components for $T$, nuisance variation, and observation noise. The encoder must learn a representation from data.

Encoders are trained under two kinds of supervision. The classification objective uses the deterministic Bayes class
\[
C=\mathbf 1\{\eta(S,T)>1/2\}=\mathbf 1\{S=+1\},
\]
which requires only the coarse quotient $S$. The probability objective uses the soft target $\eta(S,T)$, requiring the finer probability-level quotient. Each objective is trained with a tight bottleneck and a wider representation, giving class-tight, class-wide, prob-tight, and prob-wide encoders. A random-wide encoder is included as a random-feature control.

After training, each encoder is frozen. Probes are trained to recover $S$, $T$, the full state $(S,T)$, and the eta level. Since the ten values of $\eta(S,T)$ are distinct, eta-level classification and $(S,T)$ classification both measure access to the finer quotient. Results over 30 random seeds are shown in Table~\ref{tab:synthetic-learned}. All trained encoders recover $S$ almost perfectly, as expected. The important differences appear in the finer variables. A tight classification bottleneck makes $T$ essentially unrecoverable, with $T$ accuracy at chance. Probability training substantially increases eta-level and state recoverability, and widening the classification representation also increases recovery of non-required fine information.

\begin{table}[t]
\centering
\caption{Learned synthetic representations on the known-quotient distribution. Values are means with 95\% confidence intervals across 30 random seeds. The coarse quotient $S$ is recoverable from all trained encoders; finer quotient variables depend on objective and capacity.}
\label{tab:synthetic-learned}
\begin{tabular}{lccccc}
\toprule
Condition & $S$ probe & $T$ probe & $(S,T)$ probe & Eta-level probe & Eta $R^2$ \\
\midrule
Class-tight & 1.000 [1.000, 1.000] & 0.200 [0.199, 0.201] & 0.212 [0.210, 0.213] & 0.212 [0.210, 0.213] & 0.715 [0.713, 0.716] \\
Class-wide  & 1.000 [1.000, 1.000] & 0.313 [0.306, 0.320] & 0.373 [0.366, 0.379] & 0.373 [0.366, 0.379] & 0.726 [0.723, 0.729] \\
Prob-tight  & 0.999 [0.999, 1.000] & 0.276 [0.257, 0.295] & 0.477 [0.470, 0.484] & 0.477 [0.470, 0.484] & 0.845 [0.842, 0.848] \\
Prob-wide   & 1.000 [1.000, 1.000] & 0.476 [0.465, 0.487] & 0.576 [0.570, 0.583] & 0.576 [0.570, 0.583] & 0.851 [0.848, 0.854] \\
Random-wide & 0.860 [0.833, 0.886] & 0.347 [0.339, 0.355] & 0.317 [0.305, 0.329] & 0.317 [0.305, 0.329] & 0.411 [0.369, 0.454] \\
\bottomrule
\end{tabular}
\end{table}

\begin{figure}[t]
\centering
\includegraphics[width=.82\linewidth]{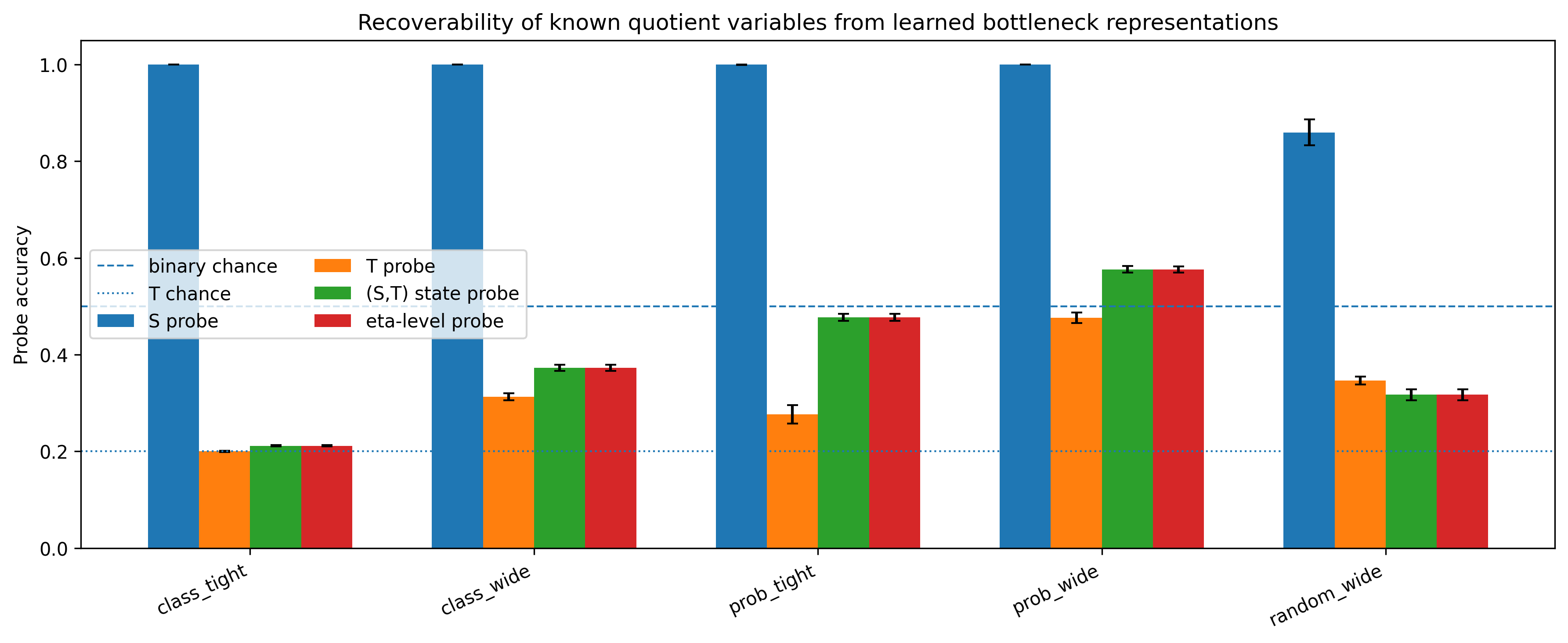}
\caption{Recoverability from learned synthetic representations. All trained encoders make the coarse quotient $S$ nearly perfectly recoverable. Finer variables $T$ and $\eta(S,T)$ are weakly recoverable from a tight classification bottleneck, more recoverable from a wide classification representation, and substantially more recoverable under probability training.}
\label{fig:synthetic-learned-probes}
\end{figure}

The paired contrasts are large and stable. Probability training improves eta-level recovery over classification training by $0.266$ at tight capacity and $0.203$ at wide capacity; both contrasts have 95\% bootstrap intervals excluding zero and paired sign-randomization $p=0.0002$. Widening the classification representation also increases eta-level recovery by $0.161$ and $T$ recovery by $0.113$, variables outside the classification objective. Class-wide and class-tight have indistinguishable $S$ recovery. These results illustrate the same principle as the exact finite calculation. All trained encoders can preserve the coarse Bayes decision quotient, and probability-trained or wider encoders preserve more of the finer probability-level quotient.

Two controls sharpen the interpretation. First, the shuffled eta-level probes remain near chance, ruling out trivial leakage through the probing protocol. Second, the random-wide encoder carries information because it is a random nonlinear map of inputs whose geometry already contains $S$ and $T$. It is treated as a random-feature baseline.

The synthetic results give both a finite and a learned view of the theory. In the finite calculation, sufficiency and minimality are known analytically. In the learned neural setting, objective and bottleneck capacity govern which known quotient variables remain recoverable. Tight classification training approaches the coarse quotient; probability training encourages preservation of the finer quotient; and wide classification training can retain additional non-required information.

\subsection{Real-data taxonomic refinement on iNaturalist}
\label{subsec:exp-inaturalist}

The preceding controlled experiments instantiate the Bayes quotients analytically. The next question is whether the refinement pattern is visible in a modern deep representation-learning setting. For this purpose, iNaturalist is used as a fine-grained natural-image dataset with a domain-native biological taxonomy. The label hierarchy has the deterministic form
\[
Y_{\mathrm{species}}
\longmapsto
Y_{\mathrm{genus}}
\longmapsto
Y_{\mathrm{family}}.
\]
A predictive distribution over species induces predictive distributions over genus and family by marginalization. This makes iNaturalist a natural real-data analogue of the quotient refinement relation studied in the theory. A species-level probabilistic target is finer than a genus-level target, which is finer than a family-level target.

This experiment tests the empirical representation-learning prediction suggested by the theory, using the iNaturalist taxonomy as a structured empirical refinement. If a representation is trained for a finer supervised decision problem, then finer taxonomic information should be more recoverable from the frozen representation. A representation trained for a coarse target may still retain non-required fine information when it is wide; a coarse bottleneck should suppress such extra information while preserving the coarse task.

A balanced taxonomic subset of iNaturalist is constructed with 40 families, 106 genera, and 277 species. Each species contributes 30 training images, 6 validation images, and 10 test images, giving 8310 training, 1662 validation, and 2770 test examples. All models use an ImageNet-pretrained ResNet-18 backbone adapted to the selected taxonomic target. Five encoders are trained, namely a wide family encoder, a bottleneck family encoder, a wide genus encoder, a wide species encoder, and a bottleneck species encoder. The wide representations have dimension 512 and the bottleneck representations have dimension 64. Two controls are also included, a frozen pretrained backbone and a random wide encoder.

After training, each encoder is frozen. Linear probes are then trained from its representation to family, genus, and species labels. For the species probe, two refinement diagnostics are also evaluated. First, the predicted species probabilities are aggregated to family probabilities using the known taxonomy. Second, within-family species accuracy is measured, with the species prediction restricted to species belonging to the true family. Finally, a shuffled-species probe is trained as a leakage/control check. All results are averaged over five random seeds.

\begin{table}[t]
\centering
\caption{Recoverability of taxonomic labels from frozen iNaturalist representations. Values are mean test accuracies over five seeds. The shuffled column reports species-probe accuracy when species labels are randomly permuted during probe training.}
\label{tab:inat-recoverability}
\begin{tabular}{lccccc}
\toprule
Representation & Family & Genus & Species & Species $\to$ family & Shuffled species \\
\midrule
Random wide & 0.203 & 0.093 & 0.049 & 0.216 & 0.003 \\
Pretrained frozen & 0.746 & 0.624 & 0.462 & 0.778 & 0.003 \\
Family bottleneck & 0.823 & 0.517 & 0.315 & 0.820 & 0.002 \\
Family wide & 0.789 & 0.618 & 0.429 & 0.819 & 0.003 \\
Genus wide & 0.772 & 0.699 & 0.491 & 0.828 & 0.003 \\
Species bottleneck & 0.769 & 0.666 & 0.515 & 0.801 & 0.003 \\
Species wide & 0.781 & 0.677 & 0.542 & 0.814 & 0.004 \\
\bottomrule
\end{tabular}
\end{table}

\begin{figure}[t]
\centering
\includegraphics[width=\linewidth]{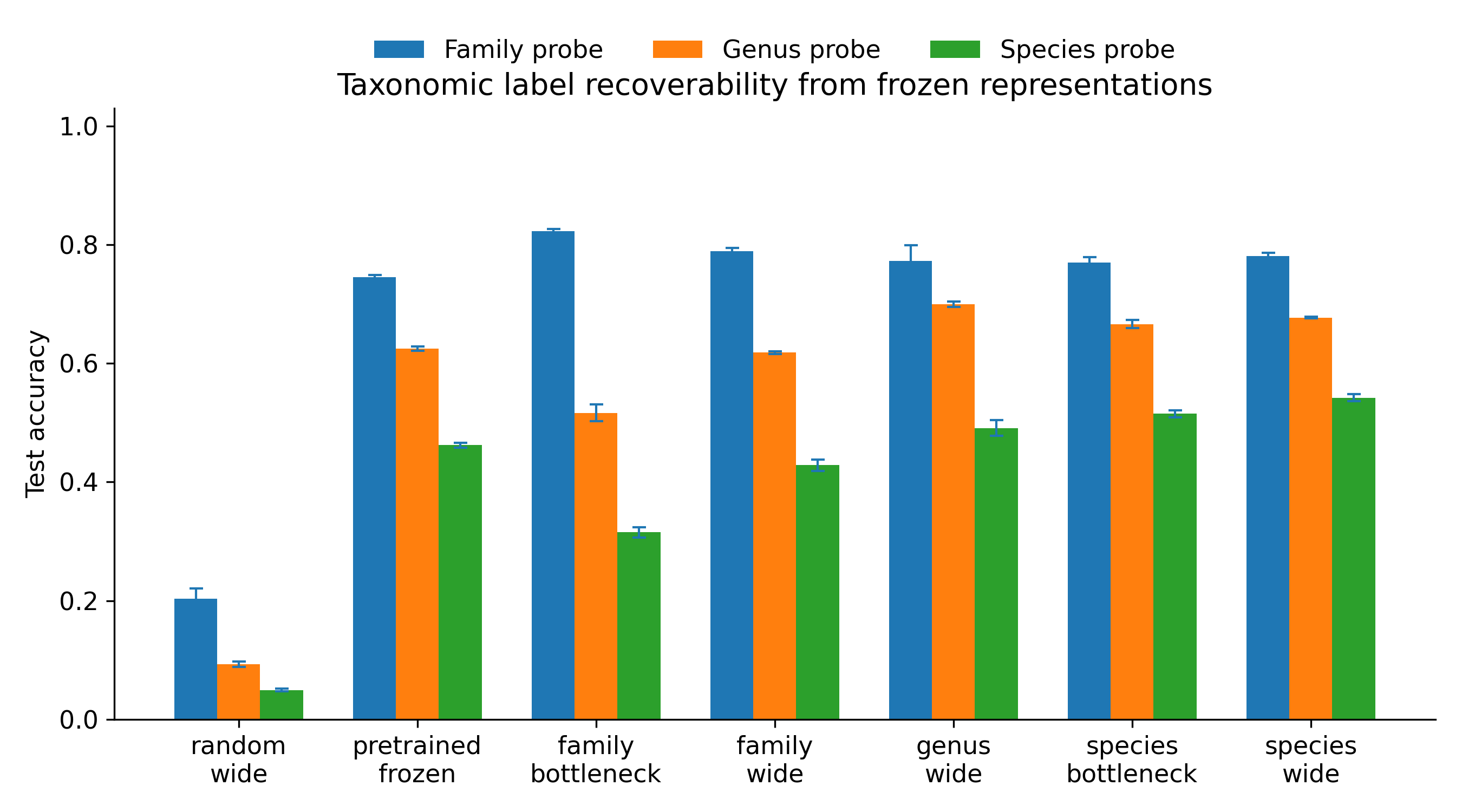}
\caption{Taxonomic label recoverability from frozen iNaturalist representations. Finer supervision increases recoverability of the corresponding finer labels, while coarse bottlenecks reduce non-required species information.}
\label{fig:inat-probes}
\end{figure}

\begin{figure}[t]
\centering
\includegraphics[width=0.72\linewidth]{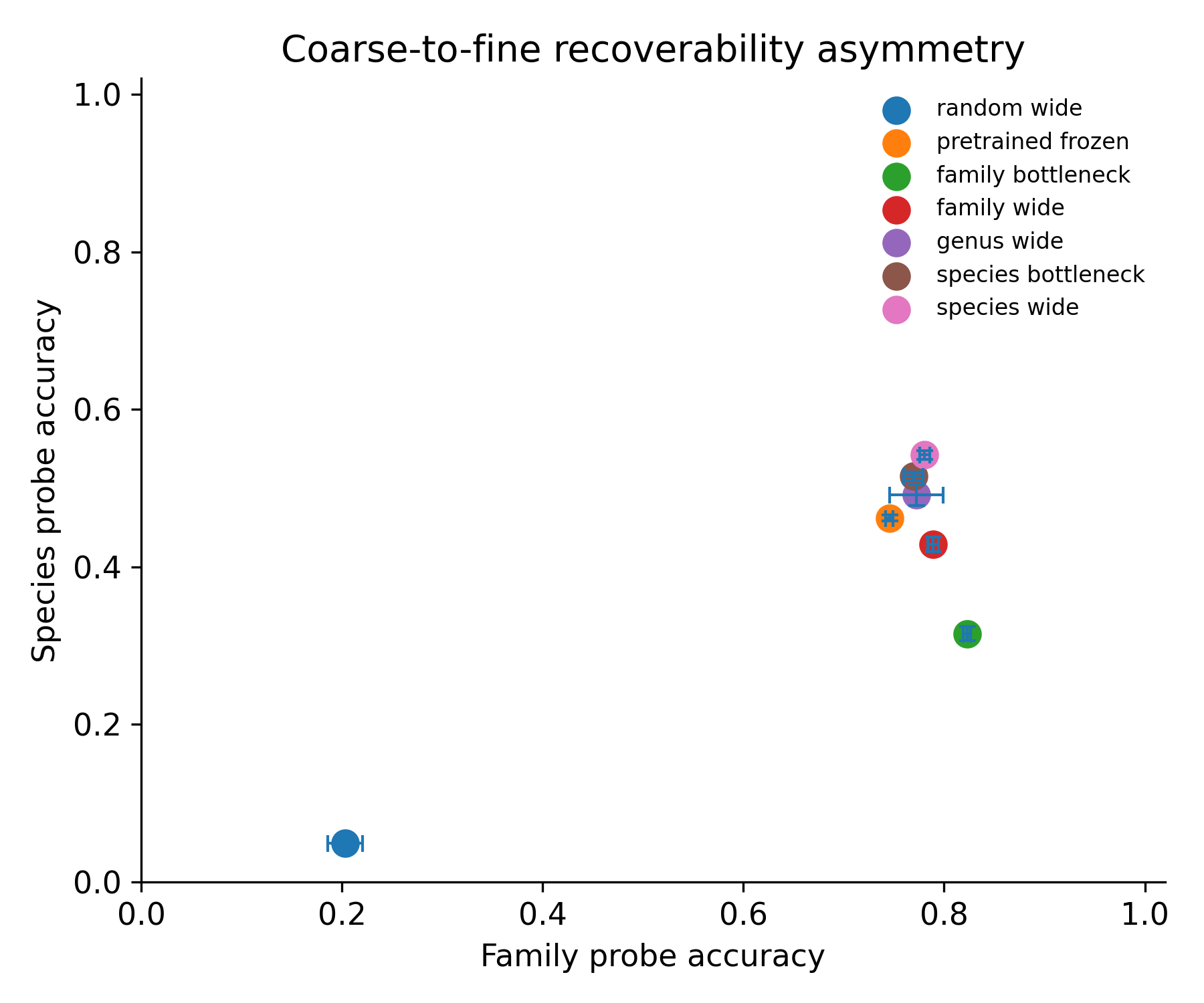}
\caption{Family--species recoverability tradeoff. The horizontal axis measures coarse family recoverability and the vertical axis measures fine species recoverability. Family bottleneck representations preserve the coarse target while discarding more species information than wide or fine-supervised representations.}
\label{fig:inat-asymmetry}
\end{figure}

\begin{figure}[t]
\centering
\includegraphics[width=\linewidth]{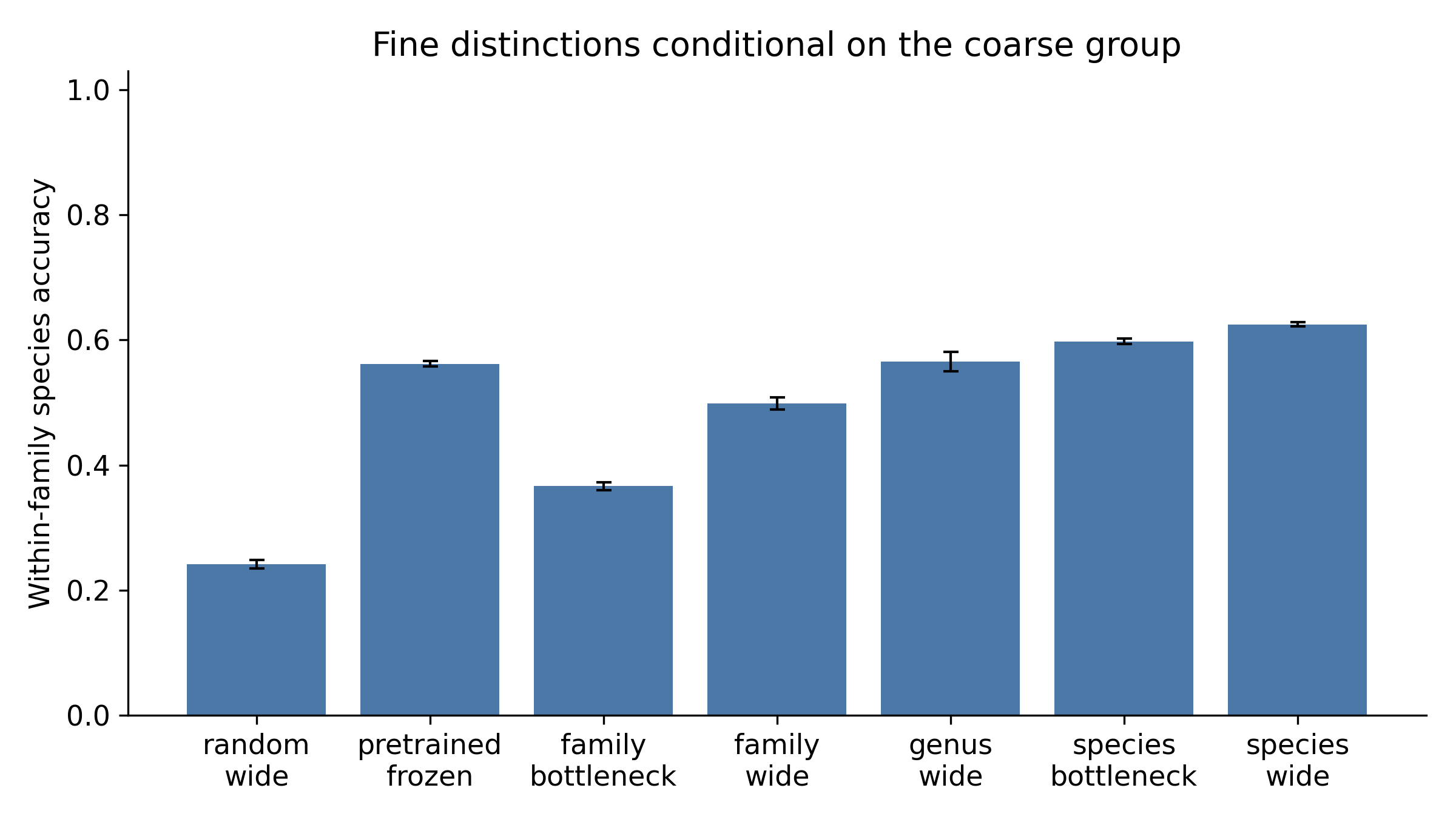}
\caption{Within-family species recovery. The true family is used only to restrict the candidate species set at evaluation time. Species-supervised representations retain the strongest fine distinctions within each coarse family.}
\label{fig:inat-within-family}
\end{figure}

\begin{table}[t]
\centering
\caption{Planned paired contrasts for the iNaturalist refinement experiment. Differences are paired across seeds. Confidence intervals are bootstrap intervals over the five paired seeds; $p$-values are paired sign-randomization/permutation values.}
\label{tab:inat-contrasts}
\begin{tabular}{lccc}
\toprule
Contrast & Mean difference & 95\% CI & $p$ \\
\midrule
Species wide $-$ family wide, species probe & 0.114 & [0.105, 0.121] & 0.028 \\
Species bottleneck $-$ family bottleneck, species probe & 0.200 & [0.196, 0.204] & 0.028 \\
Genus wide $-$ family wide, genus probe & 0.081 & [0.077, 0.085] & 0.028 \\
Family wide $-$ family bottleneck, species probe & 0.113 & [0.103, 0.121] & 0.028 \\
\bottomrule
\end{tabular}
\end{table}

The learned models first pass a basic adequacy check. Each trained encoder solves its own supervised target substantially above chance. The family encoders reach about 0.825 family accuracy, the genus encoder reaches 0.717 genus accuracy, and the species encoders reach 0.519--0.558 species accuracy, with species top-5 accuracy around 0.79--0.82. The downstream probe results are not explained by failure to learn the training targets.

The main refinement pattern is clear in Table~\ref{tab:inat-recoverability} and Figure~\ref{fig:inat-probes}. Species-supervised representations expose more species information than family-supervised representations. The species-wide encoder obtains 0.542 species probe accuracy, compared with 0.429 for the family-wide encoder. The same pattern is stronger in the bottleneck comparison. The species bottleneck obtains 0.515 species accuracy, compared with 0.315 for the family bottleneck. Similarly, the genus-wide encoder improves genus recoverability relative to the family-wide encoder, 0.699 versus 0.618. These effects are stable across seeds, as shown by the paired contrasts in Table~\ref{tab:inat-contrasts}.

The family-wide and family-bottleneck encoders expose nonminimality in the coarse task. The family-wide encoder is trained for the family target and retains substantial species information, with species probe accuracy at 0.429. The family bottleneck preserves strong family performance while reducing species recoverability to 0.315. This is the empirical analogue of a sufficient refinement: a representation may contain all information needed for a coarse Bayes action while also retaining finer, non-required distinctions.

The within-family diagnostic in Figure~\ref{fig:inat-within-family} sharpens this comparison. When the true family is supplied and the probe only has to distinguish species within that family, family-wide representations are more informative than family-bottleneck representations, while species-supervised representations are most informative. The species-wide encoder reaches 0.625 within-family species accuracy, compared with 0.498 for family wide and 0.366 for family bottleneck. The shuffled-species control remains near zero in all trained conditions, supporting the interpretation that the species probes rely on genuine taxonomic signal.

Finally, the species-to-family aggregation diagnostic confirms that fine probabilistic predictions can support the coarser target by marginalization along the taxonomy. Aggregating species-probe probabilities to family gives high family accuracy across trained representations, including 0.814 for the species-wide encoder. This is the finite-label analogue of the fact that a conditional law over a refined label determines the conditional law over a coarser deterministic function of that label.

\section{Discussion and limitations}
\label{sec:discussion}

The main implication of the theory is that predictive information is loss-dependent. A representation can be optimal for a supervised task by preserving the information needed for the Bayes action. In the unique Bayes-action setting, this information is the sigma-algebra generated by the Bayes action. Sufficiency asks for its retention; minimality asks for no more. This separates two notions that are often conflated in representation learning: being good enough for optimal prediction and being compressed to the task-relevant quotient.

The experiments should be read in this light. The finite experiment verifies the quotient structure in a setting where the population objects are known analytically. The learned synthetic experiment shows how objective choice and bottleneck capacity affect the recoverability of coarse and fine quotient variables. The iNaturalist experiment gives a real-data analogue using taxonomic refinement, where fine labels deterministically refine coarser labels and provide a practical setting in which to study sufficiency, refinement, and nonminimality.

There are several limitations. First, the theory is population-level. It characterizes exact Bayes sufficiency, leaving finite-sample estimation, optimization dynamics, and the behavior of particular training algorithms for future work. Second, probe recoverability is an empirical diagnostic. A successful probe shows that information is accessible to the probe class; a failed probe leaves open whether the information is absent or encoded beyond that probe class. Third, the clean quotient characterization relies on almost-sure uniqueness of the Bayes action. When Bayes actions are nonunique, the correct object is the common-action structure on representation fibers, which is less easily summarized by a single quotient variable. Finally, the iNaturalist hierarchy serves as an empirical analogue for natural images.

These limitations point to useful extensions. One direction is an approximate theory relating excess risk to approximate quotient recovery. Another is a finite-sample theory distinguishing statistical error, optimization error, and representation error. A third is the design of objectives or regularizers that encourage both Bayes sufficiency and approximate minimality.

\section{Conclusion}
\label{sec:conclusion}

This work formalized Bayes-sufficient representation learning as a decision-theoretic information problem. For a joint law and loss, Bayes sufficiency means factorization of some Bayes-optimal predictor through the representation. In the common case of an almost-surely unique Bayes action, this yields a canonical Bayes quotient: the informational object that identifies inputs requiring the same optimal action. The minimal predictive information is the sigma-algebra generated by the Bayes action itself.

This characterization makes the representation target explicit. For a fixed supervised problem, the distribution and the loss determine the Bayes action, the Bayes action determines the quotient, and the quotient determines the minimal information required for Bayes-optimal prediction. Different losses therefore induce different predictive quotients. Zero-one loss may require only the Bayes class, squared loss the conditional mean, Brier loss the conditional probability in binary prediction, and log loss or strictly proper scoring rules the predictive distribution.

Sufficiency is containment of the relevant quotient; minimality is equality with it. A representation can be Bayes-sufficient while retaining information that is not required by the supervised decision problem. The experiments illustrate this distinction in controlled finite distributions, learned bottleneck representations, and a real-data taxonomic refinement setting. They show how the information needed for optimal action can be separated from the additional information that a learned representation may preserve.

\bibliographystyle{plainnat}
\bibliography{references}

\clearpage

\appendix

\section{Full Bayes-quotient theory}
\label{app:bayes-quotient-theory}

This appendix gives the full formal version of Section~\ref{sec:theory}. It includes the measurable setup, the factorization characterization of Bayes sufficiency, the set-valued Bayes-action case, the unique-action quotient theory, and the standard-loss derivations.

\subsection{Measurable setup and Bayes decision problems}
\label{app:technical-setup}

Let $(\mathcal X,\mathcal A_{\mathcal X})$ and $(\mathcal Y,\mathcal A_{\mathcal Y})$ be standard Borel spaces, and let
\[
P\in\mathcal P(\mathcal X\times\mathcal Y),\qquad (X,Y)\sim P.
\]
Write $P_X$ for the marginal law of $X$. Let $(\mathsf A,\mathcal A_{\mathsf A})$ be a standard Borel action space and let
\[
\ell:\mathsf A\times\mathcal Y\to[0,\infty]
\]
be a measurable loss. The use of nonnegative extended-valued losses includes losses such as log loss with the usual convention $-\log 0=+\infty$. All risks below are understood as extended nonnegative expectations unless finiteness is explicitly stated.

A predictor is a measurable map $f:\mathcal X\to\mathsf A$, with risk
\[
R_{\ell,P}(f):=\mathbb E_P[\ell(f(X),Y)].
\]
A representation is a measurable map $h:\mathcal X\to\mathcal H$ into a standard Borel space $(\mathcal H,\mathcal A_{\mathcal H})$. Write $H=h(X)$. A head on $H$ is a measurable map $c:\mathcal H\to\mathsf A$, and the corresponding predictor is $c\circ h$.

Since the spaces are standard Borel, regular conditional laws exist. Fix a version of the conditional law
\[
\mu_x(\cdot):=P(Y\in\cdot\mid X=x).
\]
The conditional risk of action $a$ at input value $x$ is
\[
L_{\ell,P}(a\mid x):=\int_{\mathcal Y}\ell(a,y)\,\mu_x(dy),
\]
and the Bayes action correspondence is
\[
\Gamma_{\ell,P}(x):=\arg\min_{a\in\mathsf A} L_{\ell,P}(a\mid x),
\]
whenever the argmin is nonempty.

Throughout this section, whenever a result uses $\Gamma_{\ell,P}$, assume that $\Gamma_{\ell,P}(x)$ is nonempty for $P_X$-almost every $x$ and that its graph
\[
\operatorname{Gr}(\Gamma_{\ell,P})
:=
\{(x,a)\in\mathcal X\times\mathsf A:a\in\Gamma_{\ell,P}(x)\}
\]
is measurable. This ensures that statements such as $c(h(X))\in\Gamma_{\ell,P}(X)$ are measurable events. A Bayes predictor is a measurable selector of this correspondence:
\[
\mathsf{Opt}_{\ell,P}
:=
\Bigl\{
f:\mathcal X\to\mathsf A\text{ measurable}:
f(X)\in\Gamma_{\ell,P}(X)\quad P_X\text{-a.s.}
\Bigr\}.
\]
Assume $\mathsf{Opt}_{\ell,P}$ is nonempty whenever Bayes sufficiency is discussed.

All inclusions and equalities of sigma-algebras below are understood modulo $P_X$-null sets. Thus $\sigma(U)\subseteq\sigma(V)$ means that $U$ admits a $\sigma(V)$-measurable version. Equivalently, when the variables take values in standard Borel spaces, it means that $U=\phi(V)$ almost surely for some measurable map $\phi$.

\begin{definition}[Bayes sufficiency]
\label{def:bayes-sufficiency-app}
A representation $H=h(X)$ is \emph{Bayes-sufficient} for $(P,\ell)$ if there exists a measurable head $c:\mathcal H\to\mathsf A$ such that
\[
c(h(X))\in\Gamma_{\ell,P}(X)
\qquad P_X\text{-a.s.}
\]
Equivalently, at least one Bayes-optimal decision rule can be implemented from the representation.
\end{definition}

The definition is existential. In the set-valued case, Bayes sufficiency asks the representation to preserve enough information to select one Bayes-optimal action on almost every representation fiber. This is the appropriate notion for decision-theoretic sufficiency. The representation is sufficient if it supports Bayes-optimal action, even when the full set of Bayes actions contains additional possibilities.

This definition is stated through Bayes actions, which separates sufficiency from technical attainment questions. If the unrestricted Bayes risk is attained and the restricted risk over heads on $H$ is attained, then Definition~\ref{def:bayes-sufficiency-app} is equivalent to
\[
\inf_{c:\mathcal H\to\mathsf A}
\mathbb E[\ell(c(H),Y)]
=
\inf_{f:\mathcal X\to\mathsf A}
\mathbb E[\ell(f(X),Y)].
\]
The action-based formulation is the one used below.

\subsection{Bayes sufficiency as factorization}
\label{app:factorization-proof}

The first result gives the basic representation-theoretic form of Bayes sufficiency. A representation is sufficient exactly when at least one Bayes predictor factors through it.

\begin{theorem}[Bayes sufficiency iff Bayes predictor factorization]
\label{thm:factorization-app}
Assume $\mathsf{Opt}_{\ell,P}$ is nonempty. Let $h:\mathcal X\to\mathcal H$ be measurable and let $H=h(X)$. The following statements are equivalent.
\begin{enumerate}
    \item $H$ is Bayes-sufficient for $(P,\ell)$.
    \item There exists a measurable head $c:\mathcal H\to\mathsf A$ such that
    \[
    c(h(X))\in\Gamma_{\ell,P}(X)
    \qquad P_X\text{-a.s.}
    \]
    \item There exist $f^\star\in\mathsf{Opt}_{\ell,P}$ and a measurable head $c:\mathcal H\to\mathsf A$ such that
    \[
    f^\star(X)=c(H)
    \qquad P_X\text{-a.s.}
    \]
    \item At least one Bayes predictor $f^\star\in\mathsf{Opt}_{\ell,P}$ is $\sigma(H)$-measurable as a random variable $f^\star(X)$.
\end{enumerate}
\end{theorem}

\begin{proof}
The equivalence between (1) and (2) is exactly Definition~\ref{def:bayes-sufficiency-app}.

Assume (2). Define $f^\star=c\circ h$. Then $f^\star$ is measurable and satisfies
\[
f^\star(X)=c(h(X))\in\Gamma_{\ell,P}(X)
\qquad P_X\text{-a.s.}
\]
Hence $f^\star\in\mathsf{Opt}_{\ell,P}$ and (3) holds. Conversely, if (3) holds, then
\[
c(h(X))=f^\star(X)\in\Gamma_{\ell,P}(X)
\qquad P_X\text{-a.s.},
\]
so (2) holds.

If (3) holds, then $f^\star(X)=c(H)$ almost surely, so $f^\star(X)$ is $\sigma(H)$-measurable. Thus (4) holds. Conversely, suppose (4) holds. Since $f^\star(X)$ is $\sigma(H)$-measurable and both $\mathcal H$ and $\mathsf A$ are standard Borel, the Doob--Dynkin factorization lemma gives a measurable map $c:\mathcal H\to\mathsf A$ such that
\[
f^\star(X)=c(H)
\qquad P_X\text{-a.s.}
\]
Thus (3) holds.
\end{proof}

Theorem~\ref{thm:factorization-app} is elementary, and it is the point at which supervised representation learning becomes a question about information. Bayes sufficiency asserts that $H$ contains enough information to implement at least one Bayes-optimal action rule for the chosen loss.

\subsection{The nonunique case and common Bayes actions on conditional representation fibers}
\label{app:nonunique-actions}

In the presence of multiple Bayes actions, requiring all inputs merged by a representation to have identical Bayes-action sets is generally too strong. If
\[
\Gamma_{\ell,P}(x_1)=\{a,b\}
\qquad\text{and}\qquad
\Gamma_{\ell,P}(x_2)=\{a\},
\]
then the action sets differ, and a representation may merge $x_1$ and $x_2$ while maintaining Bayes optimality, because the head can choose the common action $a$. The correct condition is a common-action condition along conditional fibers of the representation.

Here the term ``fiber'' is used at the input level. It concerns the conditional law of $X$ given $H=u$. This complements the predictor-fiber viewpoint on representation--head factorizations, where fibers are sets of pairs $(h,c)$ inducing the same composite predictor.

\begin{proposition}[Common-action characterization on conditional representation fibers]
\label{prop:fiberwise-app}
Assume the hypotheses of Theorem~\ref{thm:factorization-app}. Let $K(u,\cdot)$ be a regular conditional distribution of $X$ given $H=u$. Then $H=h(X)$ is Bayes-sufficient for $(P,\ell)$ if and only if there exists a measurable map $a:\mathcal H\to\mathsf A$ such that, for $P_H$-almost every $u$,
\[
K\bigl(u,\{x\in\mathcal X:a(u)\in\Gamma_{\ell,P}(x)\}\bigr)=1.
\]
Equivalently, for almost every representation value $u$, there is a single action $a(u)$ that is Bayes-optimal for $K(u,\cdot)$-almost every input conditionally assigned to that representation value.
\end{proposition}

\begin{proof}
Suppose first that $H$ is Bayes-sufficient. By Theorem~\ref{thm:factorization-app}, there exists a measurable head $c:\mathcal H\to\mathsf A$ such that
\[
c(h(X))\in\Gamma_{\ell,P}(X)
\qquad P_X\text{-a.s.}
\]
Let
\[
S:=\{x\in\mathcal X:c(h(x))\in\Gamma_{\ell,P}(x)\}.
\]
By the measurability of the graph of $\Gamma_{\ell,P}$, the set $S$ is measurable, and by assumption $P_X(S)=1$. Since $K$ is a regular conditional distribution of $X$ given $H$,
\[
1=P_X(S)=\int_{\mathcal H}K(u,S)\,P_H(du).
\]
Hence $K(u,S)=1$ for $P_H$-almost every $u$. Also, for $P_H$-almost every $u$, the kernel $K(u,\cdot)$ is supported on the conditional fiber $\{x:h(x)=u\}$ modulo null sets. Therefore, for such $u$,
\[
K\bigl(u,\{x:c(u)\in\Gamma_{\ell,P}(x)\}\bigr)=1.
\]
Taking $a=c$ gives the required map.

Conversely, suppose there exists a measurable $a:\mathcal H\to\mathsf A$ satisfying the displayed common-action condition. Define the head $c=a$. Then, using disintegration with respect to $H$,
\[
\begin{aligned}
P_X\{x:c(h(x))\in\Gamma_{\ell,P}(x)\}
&=
\int_{\mathcal H}
K\bigl(u,\{x:c(u)\in\Gamma_{\ell,P}(x)\}\bigr)
\,P_H(du)\\
&=1.
\end{aligned}
\]
Thus $c(h(X))\in\Gamma_{\ell,P}(X)$ almost surely. By Theorem~\ref{thm:factorization-app}, $H$ is Bayes-sufficient.
\end{proof}

Proposition~\ref{prop:fiberwise-app} explains why the quotient picture is simplest under uniqueness. With nonunique Bayes actions, Bayes-sufficient representation fibers must admit a common Bayes action, while the full correspondence $x\mapsto\Gamma_{\ell,P}(x)$ may vary along those fibers. The remainder of the section therefore isolates the standard unique-action and uniquely elicited settings in which a canonical quotient is generated by a measurable statistic.

\subsection{Unique Bayes actions and Bayes quotients}
\label{app:unique-actions}

Assume now that the Bayes action is unique almost surely. That is, assume there exists a measurable map
\[
a^\star_{\ell,P}:\mathcal X\to\mathsf A
\]
such that
\[
\Gamma_{\ell,P}(X)=\{a^\star_{\ell,P}(X)\}
\qquad P_X\text{-a.s.}
\]
Then every Bayes predictor agrees with $a^\star_{\ell,P}$ almost surely. The random variable $a^\star_{\ell,P}(X)$ is the Bayes-relevant decision statistic for the loss. It induces the equivalence relation
\[
x\sim_{\ell,P}x'
\quad\Longleftrightarrow\quad
a^\star_{\ell,P}(x)=a^\star_{\ell,P}(x'),
\]
up to $P_X$-null sets. To avoid relying on a possibly inconvenient quotient space, the Bayes quotient is defined informationally as the sigma-algebra generated by this statistic.

\begin{definition}[Bayes quotient in the unique-action case]
\label{def:bayes-quotient-unique-app}
When the Bayes action is unique almost surely, the \emph{Bayes quotient} of $(P,\ell)$ is any random variable $Q_{\ell,P}=q_{\ell,P}(X)$ satisfying
\[
\sigma(Q_{\ell,P})=\sigma(a^\star_{\ell,P}(X))
\qquad \text{mod }P_X.
\]
The associated \emph{Bayes information} is the sigma-algebra
\[
\mathcal I_{\ell,P}:=\sigma(a^\star_{\ell,P}(X)).
\]
\end{definition}

The quotient is therefore meant in an informational sense. Inputs are identified when they induce the same Bayes action. The formal object is the sigma-algebra generated by the Bayes action, which avoids choosing a particular quotient space and makes the definition invariant under measurable reparameterizations of the Bayes act.

\begin{theorem}[Unique Bayes action implies quotient characterization]
\label{thm:unique-quotient-app}
Assume the Bayes action is unique $P_X$-almost surely and is represented by the measurable map $a^\star_{\ell,P}$. Let $H=h(X)$. The following statements are equivalent.
\begin{enumerate}
    \item $H$ is Bayes-sufficient for $(P,\ell)$.
    \item There exists a measurable map $c:\mathcal H\to\mathsf A$ such that
    \[
    a^\star_{\ell,P}(X)=c(H)
    \qquad P_X\text{-a.s.}
    \]
    \item $a^\star_{\ell,P}(X)$ is $\sigma(H)$-measurable.
    \item The Bayes information is contained in the representation information,
    \[
    \mathcal I_{\ell,P}
    =
    \sigma(a^\star_{\ell,P}(X))
    \subseteq
    \sigma(H)
    \qquad \text{mod }P_X.
    \]
\end{enumerate}
Consequently, in the unique-action case, Bayes-sufficient representations are exactly refinements of the Bayes quotient.
\end{theorem}

\begin{proof}
Since the Bayes action is unique almost surely, every $f^\star\in\mathsf{Opt}_{\ell,P}$ satisfies
\[
f^\star(X)=a^\star_{\ell,P}(X)
\qquad P_X\text{-a.s.}
\]
By Theorem~\ref{thm:factorization-app}, $H$ is Bayes-sufficient if and only if some Bayes predictor factors through $H$. Under uniqueness, this is equivalent to the existence of a measurable $c:\mathcal H\to\mathsf A$ such that
\[
a^\star_{\ell,P}(X)=c(H)
\qquad P_X\text{-a.s.}
\]
Thus (1) and (2) are equivalent.

Conditions (2) and (3) are equivalent by the Doob--Dynkin factorization lemma for standard Borel spaces. Condition (3) is exactly the sigma-algebra inclusion in (4).
\end{proof}

A pointwise fiber version of Theorem~\ref{thm:unique-quotient-app} may be useful for intuition. Under suitable choices of versions, Bayes sufficiency means that there is a full-measure set $E\subseteq\mathcal X$ such that
\[
h(x)=h(x')
\quad\Longrightarrow\quad
a^\star_{\ell,P}(x)=a^\star_{\ell,P}(x')
\]
for all $x,x'\in E$. Thus a Bayes-sufficient representation may refine the Bayes quotient and must keep distinct two positive-probability input regions requiring different unique Bayes actions.

\subsection{Minimal Bayes-sufficient representations}
\label{app:minimality}

The quotient characterization separates sufficiency from minimality. Sufficiency requires that the representation contain the Bayes information. Minimality requires informational equivalence with the Bayes information.

\begin{definition}[Bayes minimality]
\label{def:bayes-minimality-app}
In the unique-action setting of Theorem~\ref{thm:unique-quotient-app}, a representation $H=h(X)$ is \emph{Bayes-minimal} for $(P,\ell)$ if
\[
\sigma(H)=\sigma(a^\star_{\ell,P}(X))
\qquad \text{mod }P_X.
\]
Equivalently, $H$ is informationally equivalent to the Bayes quotient $Q_{\ell,P}$.
\end{definition}

Thus
\[
H\text{ Bayes-sufficient}
\quad\Longleftrightarrow\quad
\sigma(Q_{\ell,P})\subseteq\sigma(H),
\]
whereas
\[
H\text{ Bayes-minimal}
\quad\Longleftrightarrow\quad
\sigma(H)=\sigma(Q_{\ell,P}).
\]
This is the precise sense in which a sufficient representation may be a strict refinement of the predictive information demanded by the loss.

\begin{corollary}[Minimal Bayes-sufficient representations are unique up to information equivalence]
\label{cor:minimal-unique-app}
Assume the Bayes action is unique almost surely. If $H_1=h_1(X)$ and $H_2=h_2(X)$ are both Bayes-minimal for $(P,\ell)$, then
\[
\sigma(H_1)=\sigma(H_2)
\qquad \text{mod }P_X.
\]
In particular, Bayes-minimal representations may use different coordinates or codomains and encode the same information about $X$.
\end{corollary}

\begin{proof}
By Definition~\ref{def:bayes-minimality-app},
\[
\sigma(H_1)=\sigma(a^\star_{\ell,P}(X))
\qquad\text{and}\qquad
\sigma(H_2)=\sigma(a^\star_{\ell,P}(X))
\]
modulo $P_X$-null sets. Therefore $\sigma(H_1)=\sigma(H_2)$ modulo $P_X$-null sets.
\end{proof}

\begin{remark}[Sufficient refinements and loss of minimality]
\label{rem:refinement-app}
Assume the Bayes action is unique almost surely. If $H$ is Bayes-sufficient and $G=g(X)$ is a representation satisfying
\[
\sigma(H)\subseteq\sigma(G)
\qquad \text{mod }P_X,
\]
then $G$ is also Bayes-sufficient, because
\[
\sigma(Q_{\ell,P})\subseteq\sigma(H)\subseteq\sigma(G).
\]
Thus sufficiency is preserved under refinement.

Minimality imposes equality with the quotient. If $H$ is Bayes-minimal and $V=v(X)$ satisfies
\[
\sigma(V)\not\subseteq\sigma(Q_{\ell,P})
\qquad \text{mod }P_X,
\]
then the augmented representation $\widetilde H=(H,V)$ remains Bayes-sufficient and becomes nonminimal. Indeed, $\widetilde H$ contains $H$ and therefore contains the Bayes quotient, and it also contains information outside the quotient. This is the Bayes-quotient version of a broader factorization phenomenon. Unused representation information may be retained while predictive capability is preserved, so minimality is an additional information requirement beyond sufficiency.
\end{remark}

Remark~\ref{rem:refinement-app} describes the order structure of Bayes-sufficient representations. Bayes risk distinguishes representations that fail to contain the Bayes quotient from those that contain it; by itself, it leaves the minimal representative among sufficient refinements unspecified.

\subsection{Quotients from uniquely elicited conditional properties}
\label{app:elicitation}

The unique-action theorem applies directly to losses whose Bayes action is a uniquely elicited functional of the conditional law. This connects the representation quotient to classical elicitation theory. A loss determines which property of the predictive distribution is optimal to report, and the Bayes quotient identifies input points with the same value of that property.

Let $\mathcal P_0\subseteq\mathcal P(\mathcal Y)$ be a class of probability measures on $\mathcal Y$. Let $(\mathcal T,\mathcal A_{\mathcal T})$ be a standard Borel space, and let
\[
T:\mathcal P_0\to\mathcal T
\]
be a measurable property. For simplicity, assume the action space is $\mathsf A=\mathcal T$. A loss $\ell:\mathcal T\times\mathcal Y\to[0,\infty]$ \emph{uniquely elicits} $T$ on $\mathcal P_0$ if, for every $\mu\in\mathcal P_0$,
\[
\arg\min_{t\in\mathcal T}\int \ell(t,y)\,\mu(dy)=\{T(\mu)\}.
\]

\begin{theorem}[Uniquely elicited conditional properties induce Bayes quotients]
\label{thm:elicited-quotient-app}
Assume that $\mu_X:=P(Y\in\cdot\mid X)$ belongs to $\mathcal P_0$ almost surely and that the random variable
\[
T(\mu_X)
\]
is measurable. If $\ell$ uniquely elicits $T$ on $\mathcal P_0$, then the Bayes action is unique almost surely and is given by
\[
a^\star_{\ell,P}(X)=T(P(Y\in\cdot\mid X)).
\]
Consequently, a representation $H=h(X)$ is Bayes-sufficient for $(P,\ell)$ if and only if
\[
\sigma\bigl(T(P(Y\in\cdot\mid X))\bigr)
\subseteq
\sigma(H)
\qquad \text{mod }P_X.
\]
Moreover, $H$ is Bayes-minimal if and only if
\[
\sigma(H)=\sigma\bigl(T(P(Y\in\cdot\mid X))\bigr)
\qquad \text{mod }P_X.
\]
\end{theorem}

\begin{proof}
For each $x$ such that $\mu_x\in\mathcal P_0$, the conditional risk is
\[
L_{\ell,P}(t\mid x)=\int \ell(t,y)\,\mu_x(dy).
\]
Since $\ell$ uniquely elicits $T$ on $\mathcal P_0$, the unique minimizer of this conditional risk is $T(\mu_x)$. Thus
\[
\Gamma_{\ell,P}(x)=\{T(\mu_x)\}
\]
for $P_X$-almost every $x$. The assumed measurability of $T(\mu_X)$ gives a measurable unique Bayes action
\[
a^\star_{\ell,P}(X)=T(\mu_X)=T(P(Y\in\cdot\mid X)).
\]
The sufficiency and minimality characterizations now follow immediately from Theorem~\ref{thm:unique-quotient-app} and Definition~\ref{def:bayes-minimality-app}.
\end{proof}

Theorem~\ref{thm:elicited-quotient-app} gives a general recipe. Once a loss uniquely elicits a conditional property, the representation target is the input quotient generated by that conditional property. The quotient arises inside the supervised problem as the representation-level form of the Bayes act elicited by the loss.

\subsection{Standard supervised losses}
\label{app:standard-losses}

The preceding theorem recovers the usual supervised examples. The examples also show why the representation target is loss-dependent. Different losses applied to the same joint law may require different functions of the conditional distribution.

\begin{corollary}[Standard losses instantiate the quotient framework]
\label{cor:standard-losses-app}
The following losses admit Bayes quotients in the sense of Theorem~\ref{thm:unique-quotient-app} or Theorem~\ref{thm:elicited-quotient-app}.
\begin{enumerate}
    \item \textbf{Zero-one classification.} Suppose $\mathcal Y=\{1,\dots,K\}$, $\mathsf A=\mathcal Y$, and
    \[
    \ell_{0/1}(a,y)=\mathbf 1\{a\neq y\}.
    \]
    If the conditional class probabilities have a unique maximizer almost surely, then
    \[
    a^\star_{0/1,P}(X)=\arg\max_{k\in\{1,\dots,K\}}P(Y=k\mid X),
    \]
    and the Bayes quotient is generated by the Bayes class. In the binary case, with $\eta(X)=P(Y=1\mid X)$ and $P_X(\eta(X)=1/2)=0$,
    \[
    \sigma(Q_{0/1,P})=\sigma\bigl(\mathbf 1\{\eta(X)>1/2\}\bigr).
    \]

    \item \textbf{Squared loss.} Suppose $\mathcal Y=\mathsf A=\mathbb R^d$, $\ell(a,y)=\|a-y\|_2^2$, and $\mathbb E\|Y\|_2^2<\infty$. Then
    \[
    a^\star_{\mathrm{sq},P}(X)=\mathbb E[Y\mid X],
    \]
    and the Bayes quotient is generated by the conditional mean.

    \item \textbf{Binary Brier loss.} Suppose $\mathcal Y=\{0,1\}$, $\mathsf A=[0,1]$, and
    \[
    \ell_{\mathrm{Brier}}(p,y)=(p-y)^2.
    \]
    Then the unique Bayes action is
    \[
    p^\star(X)=P(Y=1\mid X),
    \]
    and the Bayes quotient is generated by the conditional probability $\eta(X)=P(Y=1\mid X)$.

    \item \textbf{Finite-label log loss.} Suppose $\mathcal Y=\{1,\dots,K\}$ and the action space is the probability simplex $\Delta_K$. Let
    \[
    \ell_{\log}(p,y)=-\log p_y,
    \]
    with $-\log 0=+\infty$. Then the unique Bayes action is the conditional probability vector
    \[
    \pi_X=(P(Y=1\mid X),\dots,P(Y=K\mid X)),
    \]
    and the Bayes quotient is generated by $P(Y\mid X)$. Thus log-loss Bayes sufficiency coincides with conditional-law sufficiency in the finite-label setting.

    \item \textbf{Strictly proper scoring rules.} More generally, let the action be a predictive distribution and let $S(q,y)$ be a strictly proper scoring rule on a class $\mathcal P_0$ containing $P(Y\in\cdot\mid X)$ almost surely. Under the usual measurability and integrability conditions ensuring that conditional expected scores are well-defined, the unique Bayes action is
    \[
    q^\star_X=P(Y\in\cdot\mid X),
    \]
    and the Bayes quotient is generated by the conditional law.
\end{enumerate}
\end{corollary}

\begin{proof}
For zero-one classification, the conditional risk of action $a\in\mathcal Y$ is
\[
L_{0/1,P}(a\mid X)=P(Y\neq a\mid X)=1-P(Y=a\mid X).
\]
Minimizing this risk is equivalent to maximizing $P(Y=a\mid X)$. Under the assumed absence of ties, the maximizer is unique almost surely, so Theorem~\ref{thm:unique-quotient-app} applies. The binary expression follows by writing $\eta(X)=P(Y=1\mid X)$.

For squared loss, expand the conditional risk. For any $a\in\mathbb R^d$,
\[
\mathbb E[\|a-Y\|_2^2\mid X]
=
\|a-\mathbb E[Y\mid X]\|_2^2
+
\mathbb E[\|Y-\mathbb E[Y\mid X]\|_2^2\mid X],
\]
where the second term is independent of $a$. Hence the unique minimizer is $\mathbb E[Y\mid X]$.

For binary Brier loss, write $\eta(X)=P(Y=1\mid X)$. The conditional risk of reporting $p\in[0,1]$ is
\[
\mathbb E[(p-Y)^2\mid X]
=
(p-\eta(X))^2+\eta(X)(1-\eta(X)).
\]
Thus the unique minimizer is $p=\eta(X)$.

For finite-label log loss, write $\pi_X(k)=P(Y=k\mid X)$. The conditional risk of reporting $p\in\Delta_K$ is
\[
L_{\log,P}(p\mid X)
=
-\sum_{k=1}^K \pi_X(k)\log p_k.
\]
Using the identity
\[
-\sum_{k=1}^K \pi_X(k)\log p_k
=
-\sum_{k=1}^K \pi_X(k)\log \pi_X(k)
+
D_{\mathrm{KL}}(\pi_X\|p),
\]
with the standard conventions for zero probabilities, this risk is uniquely minimized at $p=\pi_X$. Hence the Bayes action is the conditional probability vector, and the quotient is generated by $P(Y\mid X)$.

For strictly proper scoring rules, strict propriety means that for every conditional law $\mu\in\mathcal P_0$, the expected score $q\mapsto\int S(q,y)\,\mu(dy)$ is uniquely minimized at $q=\mu$. Applying Theorem~\ref{thm:elicited-quotient-app} with $T$ equal to the identity property on $\mathcal P_0$ gives the stated quotient.
\end{proof}

The corollary highlights the main conceptual consequence. Each loss selects its own predictive property of the conditional distribution. Log loss and strictly proper scoring rules select the conditional law, zero-one loss selects the Bayes class, squared loss selects the conditional mean, and binary Brier loss selects the conditional probability. Thus the minimal predictive representation is determined jointly by the distribution and the loss.

\section{Reproducibility}
The experiment code is organized as a resumable staged pipeline. The stages are dataset preparation, encoder training, frozen-feature extraction, probe training, statistical analysis, figure generation, and archiving. The exact taxonomic subset is fixed by a recorded random seed and saved as train/validation/test record files, together with the species--genus--family mapping. The run analyzed here uses five seeds, $\{1001,\ldots,1005\}$, ImageNet-pretrained ResNet-18 backbones, SGD with momentum 0.9, weight decay $10^{-4}$, batch size 128, cosine learning-rate decay, and early stopping on validation cross-entropy. Probe heads are linear classifiers trained on cached frozen features with standardized coordinates. Reported confidence intervals are computed by bootstrap resampling over paired seeds, and contrasts use paired sign randomization/permutation tests. 
\end{document}